\documentclass{IEEEtran}
\usepackage{amssymb}
\usepackage{nonfloat}
\usepackage{subfig}
\usepackage{multirow}
\usepackage{graphicx}
\usepackage{epstopdf}
\usepackage{amsmath}
\usepackage{amsthm}

\newtheorem{myth}{Theorem}
\newtheorem{mylem}{Lemma}

\newtheorem{mydef}{Definition}
\newtheorem{myprop}{Proposition}

\begin{document}

\title{The Application of Differential Privacy for Rank Aggregation: Privacy and Accuracy }
%
%
%
%
%

%

\author{
\IEEEauthorblockN{Shang Shang,  Tiance Wang, Paul Cuff, and Sanjeev Kulkarni}\\
\IEEEauthorblockA{Department of Electrical Engineering, 
Princeton University, Princeton NJ, 08540, U.S.A.
} \\
\{sshang, tiancew, cuff, kulkarni\}@princeton.edu
}

%
%


\maketitle

\begin{abstract}
The potential risk of privacy leakage prevents users from sharing their honest opinions on social platforms. This paper addresses the problem of privacy preservation if the query returns the histogram of rankings. The framework of differential privacy is applied to rank aggregation. The error probability of the aggregated ranking is analyzed as a result of noise added in order to achieve differential privacy. Upper bounds on the error rates for any positional ranking rule are derived under the assumption that profiles are uniformly distributed. Simulation results are provided to validate the probabilistic analysis.
\end{abstract}



\section{Introduction}
\label{sec:introduction}

\noindent With the increasing interest in social networks and the availability of large datasets, rank aggregation has been studied intensively in the context of social choice. From the NBA's Most Valuable Player to Netflix's movie recommendations, from web search to presidential elections, voting and ranking are ubiquitous. Informally, rank aggregation is the problem of combining a set of full or partial rankings of a set of alternatives into a single consensus ranking. In recommender systems, users are motivated to submit their rankings in order to receive personalized services. On the other hand, they may also be concerned about the risk of possible privacy leakage. 

Even accumulated or anonymized datasets are not as ``safe'' as they seem to be. Information on individual rankings or preferences can still be learned even if the querier only has access to global statistics. In 2006, Netflix launched a data competition with 100 million movie ratings from half a million anonymized users. However, researchers subsequently demonstrated that individual users from this ``sanitized" dataset could be identified by matching with the Internet Movie Database (IMDb). This raises the privacy concerns about sharing honest opinions. 

Differential privacy is a framework that aims to obscure individuals' appearances in the database. It makes no assumptions on the attacker's background knowledge. Mathematical guarantees are provided in \cite{dwork2006differential} and \cite{dwork2006our}. Differential privacy has gained popularity in various applications, such as social networks \cite{task2012guide}, recommendations \cite{mcsherry2009differentially}, advertising \cite{lindell2011practical}, etc. However, there is a trade-off between the accuracy of the query results and the privacy of the individuals included in the statistics. In \cite{machanavajjhala2011personalized}, the authors showed that good private social recommendations are achievable only for a small subset of users in the social network.

In this paper, we apply the framework of differential privacy to rank aggregation. Privacy is protected by adding noise to the query of ranking histograms. The user can then apply a rank aggregation rule to the ``noisy'' query results. In general, stronger noise guarantees better differential privacy. However, excessive noise reduces the utility of the query results. We measure the utility by the probability that the aggregated ranking is accurate. A summary of the contributions of this paper is as follows: 

\begin{itemize}
\item A privacy-preserving algorithm for rank aggregation is proposed. Instead of designing differential privacy for each individual ranking rule, we propose to add noise to the ranking histogram, irrespective of the ranking rules to be used.
\item General upper bounds on the ranking error rate are derived for all positional ranking rules. Moreover, we show that the asymptotic error rate approaches zero when the number of voters goes to infinity for any ranking rules with a fixed number of candidates.  
\item An example using Borda count is given to show how to extend the proposed analysis to derive a tighter upper bound on the error rate for a specific positional rule. Simulations are performed to validate the analysis.
\end{itemize}

The rest of the paper is organized as follows. We define the problem of rank aggregation, introduce the definition of differential privacy, and describe the privacy preserving algorithm in Section 2. We then discuss the accuracy of the algorithm, and provide analytical upper bounds on the error rates in Section 3, followed by simulation results in Section 4, and conclusions in Section 5. 

\section{Differential Privacy in Rank Aggregation}
\subsection{Rank Aggregation: Definitions and Notations}
\label{sub: rank}
\noindent Let $\mathcal{C}= \{1,...,M\}$ be a finite set of $M$ candidates, $M \ge 3$. Denote the set of permutations on $\mathcal{C}$ by $T_M$. Denote the number of voters by $N$. Each ballot $x_i, i = 1, . . . , N$ is an element of $T_M$, or a strict linear ordering. A rank aggregation algorithm, or a ranking rule is a function $g : T_M^N \rightarrow T_M$. The input $(x_1,\dots,x_N)$ is called a \emph{profile}. 

A ranking rule $g$ is \emph{neutral} if it commutes with permutations on $\mathcal{C}$ \cite{kalai2002fourier}. Intuitively, a neutral ranking method is not biased in favor of or against any candidate. 

A ranking rule $g$ is \emph{anonymous} if the ``names'' of the voters do not matter \cite{kalai2002fourier}, i.e. 
\begin{equation}
g(x_1,...,x_N) = g(\pi(x_1,...,x_N))
\end{equation}
for any permutation $\pi$ on $1,...,N$. For an anonymous ranking method, we use the \emph{anonymized profile}, a vector $q \in \mathbb{N}^{M!}$, instead of the complete profile ($x_1,\dots,x_N$) as the input. Let $q$ denote the histogram of rankings: It counts the number of appearances of each ranking in all $n$ rankings. The rank aggregation function can therefore be rewritten as $g: \mathbb{N}^{M!} \rightarrow T_M$.

An anonymous ranking rule is \emph{scale invariant} if the output depends only on the empirical distribution of votes $v = q/N$, not the number of voters $N$. That is, 
\begin{equation}
g(q) = g(\alpha q)
\end{equation}
for any $\alpha > 0$. 

There are many different neutral and scale invariant rank aggregation algorithms. Popular ones include plurality, Borda count, instant run-off, the Kemeny-Young method and so on. Each algorithm has its own merits and disadvantages. For example, the Kemeny-Young method satisfies the Condorcet criterion (a candidate preferred to any other candidate by a strict majority of voters must be ranked first) but is computationally expensive. In fact it is NP-Hard even for $M=4$ \cite{conitzer2006improved}. This is especially an issue for recommender systems since the number of items to be recommended can be large. 

A class of ranking rules, known as the positional rules, has an edge in computational complexity. A positional rule takes complete rankings as input, and assigns a score to each candidate according to their position in a ranking. The candidates are sorted by their total scores summed up from all rankings. The time complexity is only $O(MN + M\log M)$, where the $M\log M$ term comes from sorting. All positional rules satisfy anonymity and neutrality but fail the Condorcet criterion \cite{dwork2001rank}. A positional rule with $M$ candidates has $M$ parameters: $ s_1 \geq \dots \geq s_M$, where $s_i$ is the score assigned to the $i$th highest-ranked candidate. We can further normalize the scores without affecting the ranking rule so that $s_1=1, s_M=0$. Borda count, a widely used positional rule,  is specified by $s_i = (M-i)/(M-1)$. Note that plurality is a positional rule with $s_i = 0$ for $i \geq 2$. Plurality is popular due to its simplicity. However, it is not ideal as a rank aggregation algorithm because it discards too much information. In this paper, we specifically focus on positional rules because of their computational efficiency and ease of error rate analysis.

\subsection{Differential Privacy}
\noindent In this paper, we consider a strong notion of privacy, \emph{differential privacy} \cite{dwork2006differential}. Intuitively, a randomized algorithm has good differential privacy if its output distribution is not sensitive to a single entity's information. For any dataset $A$, let $\mathcal{N}(A)$ denote the set of neighboring datasets, each differing from $A$ by at most one record, i.e., if $A'\in \mathcal{N}(A)$, then $A'$ has exactly one entry more or one entry less than $A$.

\begin{mydef}
\cite{dwork2006our} A random algorithm $\mathcal{M}$ satisfies $(\epsilon, \delta)$-differential privacy if for any neighboring datasets $A$ and $A'$, and any subset $S$ of possible outcomes Range($\mathcal{M}$),
\begin{equation}
\Pr[\mathcal{M}(A)\in S] \le \exp(\epsilon)\times \Pr[\mathcal{M}(A')\in S]+ \delta.
\end{equation}
\end{mydef}
\emph{Remark}: $(\epsilon, \delta)$-differential privacy is a slight relaxation from the $\epsilon$-differential privacy in that the ratio  $$\Pr[\mathcal{M}(A)\in S]/\Pr[\mathcal{M}(A')\in S]$$ need not be bounded if both probabilities are very small. Differential privacy has been widely used in various applications \cite{mcsherry2009differentially,lindell2011practical}.

\subsection{Privacy Preserving Algorithms}
\label{sec:alg}
\noindent Much work has been done on developing differentially private algorithms \cite{dwork2006calibrating,barak2007privacy}. Let $\mathcal{D}$ denote the set of all  datasets, and $f$ is an operation on the dataset, such as $sum$, $count$, etc.
\begin{mydef}
The $l_2$-sensitivity $\Delta f$ of a function $f:\mathcal{D}\rightarrow\mathbb{R}^{d}\,$ is 
$$\Delta f(A)=\max_{A' \in \mathcal{N}(A)} \lVert f(A)-f(A') \rVert_{2}\,$$
for all  $A' \in \mathcal{N}(A)$ differing in at most one element, and $A, A'\in\mathcal{D}$.
\end{mydef}


\begin{myth}
\cite{dwork2006our} Define $\mathcal{M}(A)$ to be $f(A)+\mathcal{N}(0,\sigma^2I_{d\times d})$. $\mathcal{M}$ provides $(\epsilon, \delta)$-differential privacy, whenever
\begin{equation}
\sigma ^2 \ge \frac{2\ln(\frac{2}{\delta})}{\epsilon ^2} \cdot \max_{A' \in \mathcal{N}(A)} \lVert  f(A)-f(A')\rVert _2^2,
\end{equation}
for all  $A' \in \mathcal{N}(A)$ differing in at most one element, and $A, A'\in\mathcal{D}$.
\end{myth}
In our model, $f(A)$ is the histogram of all rankings, i.e. the input vector $q$ defined in Section \ref{sub: rank}. It is clear that the  $l_2$ sensitivity of $f(A)$ is 1, since adding or removing a vote can only affect one element of $q$ by 1. In the exposition, we will denote the private data and released data by $x$ and $\hat{x}$ respectively. When we add noise $n$ to a variable $x$, we write
$\hat{x} = x + noise.$ 
Thus 
\begin{equation}
\hat{q} = q + \mathcal{N}(0,\sigma^2I_{M!\times M!})
\end{equation}
where $\sigma^2 = 2\ln(\frac{2}{\delta})/\epsilon ^2$, and $M$ is the number of candidates. We use Gaussian instead of Laplacian noise which achieves stronger $\epsilon$-privacy \cite{dwork2006differential}, because Gaussian noise enjoys the nice property that any linear combination of jointly Gaussian random variables is Gaussian. 

Note that there is a positive probability that $\hat{q}_i < 0$ for some index $i$. This does not harm our analysis since positional rules are well defined even if we allow negative vote counts. 

Finally, we define the error rate of a privacy preserving rank aggregation algorithm on ranking. The error rate is the probability that the aggregated ranking changes after adding noise. This probability depends on the ranking rule, the noise distribution, and the distribution of profiles. 
\begin{mydef}
The error rate $P_e^M$ of a privacy preserving rank aggregation algorithm $g$ with $M$ candidates is defined as 
$\mathbb{E}\bf{1}_{\{ g(q)\ne g(\hat{q})\}}$.
\end{mydef}

\section{General Error Bounds}
\label{sec:main}
\noindent In this section, we discuss the error rates in the rank aggregation problem. We give the expression for the general error rate and derive upper bounds on the error rate for all positional ranking rules under the assumption that profiles are uniformly distributed.  

\subsection{Geometric Perspective of Positional Ranking Systems}
\label{sec:linear}
\noindent We normalize the anonymous profile by dividing by the number of voters $N$. The resulting vector $v = q/N$ is the empirical distribution of votes, $v \in [0,1]^{M!}$. All empirical distributions are contained in a unit simplex, called the \emph{rank simplex}: 
\begin{equation}
\label{eqn:simplex}
\mathcal{V} = \{v \in \mathbb{R}^{M!}: \sum_{i=1}^{M!} v_i = 1 \text{ and } v_i \ge 0 \text{ for } \forall i \}.
\end{equation}
A rank simplex with $M$ candidates has a dimension of $M!-1$. We assume that the normalized profile $v$ is uniformly distributed on the rank simplex $\mathcal{V}$. 

Geometrically, a ranking rule is a partition of the rank simplex. For positional ranking rules, the rank simplex is partitioned into $M!$ congruent polytopes by $M \choose 2$ hyperplanes. Each polytope represents a ranking, and each hyperplane represents the equality of the score of two candidates. Moreover, each polytope is uniquely defined by $M-1$ hyperplanes and the faces of the rank simplex $\mathcal{V}$. An example of how to define the hyperplane from given ranking rule will be given in Section \ref{sec:example}.

To maintain neutrality, we break ties randomly when there is a tie. For example, if the score of candidate $a$ and $b$ happens to be equal, then we rank $a$ ahead of $b$ with probability one half. We only mention tie as a side remark since it does not have an affect on the probability analysis.

\begin{myprop}
Let 
\begin{equation}
\hat{v} = v + \omega
\end{equation}
where $\omega$ is a $M!$-dimensional random variable with distribution $$\mathcal{N}(0, \hat{\sigma}^2I_{M!\times M!}),$$ where $\hat{\sigma}^2 = \frac{2\ln{(2/\delta)}}{\epsilon^2N ^2}$. We have 
$$\mathbb{E}\bf{1}_{\{ g(q)\ne g(\hat{q})\}} = \mathbb{E}\bf{1}_{\{ g(v)\ne g(\hat{v})\}}.$$
\end{myprop}
\begin{proof}
This follows directly from the scale invariant property of the ranking rules.
\end{proof}
\emph{Remark}: Note that $\hat{v}$ may not be in the probability simplex. The ranking result of $\hat{v}$ is uniquely defined by the cone formed by $M-1$ hyperplanes representing the equality of scores of two candidates.


\subsection{An Upper Bound on the General Error Rate}
\noindent Rather than providing different upper bounds for each and every positional rule, we derive a general bound that works for any positional rule. Therefore, the user can decide which positional rule to apply to the queried noisy histogram, and the system has some guarantee on the error rate given the privacy level. 

If noise switches the order of the scores of any two candidates, then the final ranking necessarily changes. Let $S_i(v)$, $S_j(v)$ denote the score of candidate $i$ and $j$ for an arbitrary positional rule given the profile $v$. As mentioned in Section \ref{sec:linear}, there are $M \choose 2$ hyperplanes separating the simplex into $M!$ polytopes. The hyperplanes are defined by $S_i = S_j$ for any pair of candidates $i,j$, and there are $M \choose 2$ such pairs. Let $\beta_{ij}$ denote the unit normal vector of hyperplane $\mathcal{H}_{ij}: S_i  =  S_j $. That is,
\begin{equation}\label{eqn:beta}
||\beta_{ij}||_2 = 1
\end{equation}

Then $\beta_{ij}\cdot w$ is the scalar projection of $\beta_{ij}$ for vector $w$. Let $D_{ij}(v)$ be the distance from $v$ to hyperplane $\mathcal{H}_{ij}$. Given the uniform distribution of $v$ over the rank simplex, $D_{ij}(v)$ is a continuous random variable that takes values on $[-\sqrt{2},\sqrt{2}]$ ($\sqrt{2}$ is the edge length of the probability simplex). The sign indicates on which side of the hyperplane $v$ locates. Let $p_D$ denote the probability density function of $D_{ij}$. By the neutrality of positional rules, $p_D$ is identical for any $i \neq j$ and $p_D(l) = p_D(-l)$. By symmetry,
\begin{equation} 
\int_0^{\sqrt{2}} p_D(l) \text{d}l = \frac{1}{2}.
\end{equation}
Geometrically, $p_D(l)$ is proportional to the $(M!-2)$-measure of the cross section of the hyperplane $\mathcal{H}_{ij}(l)$ with the simplex, where $\mathcal{H}_{ij}(l)$ is parallel to $\mathcal{H}_{ij}$ with distance $l$. 
\begin{mylem}
\label{lem:decrease}
Let $p_D$ be as defined as above. Then $p_D$ is maximal at 0 on $[0, \sqrt{2}]$ for any positional rule.
\label{lem:pddecreasing}
\end{mylem}

\begin{proof}
Let $\mathcal{H}$ be the hyperplane defined by the equality of the score of two candidates for an arbitrary positional rule, and $\beta$ be the unit normal vector of $\mathcal{H}$. That is, $\mathcal{H} = \{ v\in R^{M!} : \beta v = 0 \}$. Let $\mathcal{H}+s\beta$ denote the hyperplane $\beta v = s$. Let $X_1,\dots,X_{M!}$ be i.i.d. random variables with the following density function:
\begin{equation}
f(x) = 
\begin{cases}
e^{-x} & \mbox{ if } x \geq 0 \\
0 & \mbox{ otherwise}.
\end{cases}
\end{equation}
That is, $X_j$'s are independent exponential random variables with parameter $\lambda = 1$. The density of the random variable $Y = \sum_{i=1}^{M!} \beta_j X_j$ is \cite{webb1996central}
\begin{equation}
G(s) = \int_{\mathcal{H} + s\beta} \prod_{j=1}^{M!} f(x) \mbox{dVol}_\mathcal{H}
\end{equation}
where $\mbox{Vol}_\mathcal{H}$ denotes the Lebesgue measure on $\mathcal{H}$. It is shown in \cite{webb1996central} that 

\begin{equation}
\mbox{Vol}_{M!-2}(\mathcal{H} \cap \mathcal{V}) = \frac{\sqrt{M!}}{\Gamma(M!-1)} \int_H \prod_{j=1}^{M!} f(x) \mbox{dVol}_\mathcal{H}  
\end{equation}
where $\mbox{Vol}_{M!-2}$ denotes $M!-2$ - dimensional volume, $\mathcal{V}$ is the unit regular $M!-1$ - simplex embedded in $R^{M!}$, as defined in Equation (\ref{eqn:simplex}).  This result is shown in  \cite{webb1996central} for $\mathcal{H}$ passing through the origin and centroid, but it holds for any hyperplane, i.e.,
\begin{equation}
\mbox{Vol}_{M!-2}\big((\mathcal{H}+s\beta) \cap \mathcal{V}\big) = \frac{\sqrt{M!}}{\Gamma(M!-1)}G(s).
\end{equation}

The characteristic function of $Y$ is
\begin{equation}
\phi_Y(t) = \prod_{j=1}^{M!} \phi_{X_j}(\beta_j t) = \prod_{j=1}^{M!} (1+i \beta_j t)^{-1}.
\end{equation}
Note that for any entry $j$, there is a corresponding entry $j'$ such that the $j'$th ranking is the reversed order of the $j$th ranking. By symmetry, $\beta_j = -\beta_{j'}$, $(1+i \beta_j t)(1+i \beta_{j'}t) = 1+\beta_j^2 t^2$. Without loss of generality, suppose $\beta_j >0 $ for $1 \leq j \leq M!/2$, then
\begin{equation}
\phi_Y(t) = \prod_{j=1}^{M!/2} (1+ \beta_j^2 t^2)^{-1}.
\end{equation}

Since $\phi_Y(t)$ is always real and positive, by Bochner's theorem \cite{bochner1959lectures}, $G(s)$ is a positive-definite function, i.e.,
$$|G(s)|\le G(0).$$
This is also easy to prove by directly applying the inverse Fourier Transform:
\begin{align}
|G(s)| &= \left| \frac{1}{2\pi} \int_{-\infty}^{+\infty} \phi_Y(t) e^{-ist} \mbox{d}s \right|	\notag\\
& \leq  \frac{1}{2\pi} \int_{-\infty}^{+\infty} \left| \phi_Y(t) e^{-ist} \right| \mbox{d}s	\notag\\
& =  \frac{1}{2\pi} \int_{-\infty}^{+\infty} \phi_Y(t) \left| e^{-ist} \right| \mbox{d}s	\notag\\
& =  \frac{1}{2\pi} \int_{-\infty}^{+\infty} \phi_Y(t) \mbox{d}s	\notag\\
& = G(0).
\end{align}

Thus we have,
$$\mbox{Vol}_{M!-2}\big((\mathcal{H}+s\beta) \cap \mathcal{V}\big) \le \mbox{Vol}_{M!-2}(\mathcal{H} \cap \mathcal{V}).$$
\end{proof}
\label{thm:halflem}

\begin{mylem}
\label{thm:halflem}
The ranking error rate $P_e^M$  satisfies 
$$
P_e^M  \le  \binom{M}{2} \cdot  2\int\limits_0^\tau p_D(l)Q\left( \frac{l}{\hat{\sigma} }\right) \text{d}l + Q\left( \frac{\tau}{\hat{\sigma} }\right),	\forall \tau>0,
$$
for all positional ranking aggregation algorithms with $M$ candidates and $N$ voters, taking input from the $(\epsilon, \delta)$-differentially private system defined in Section \ref{sec:alg}.
\end{mylem}

\begin{proof}
The main idea of the proof is as follows. Divide the rank simplex into two parts: a ``high error'' region, denoted as $\mathcal{R}_H$, and a ``low error'' region, denoted as $\mathcal{R}_L$, as shown in Figure \ref{fig:example}. $\mathcal{R}_H$ consists of the thin slices of the simplex close to the boundary hyperplanes. $\mathcal{R}_L$ occupies most of the simplex, but $P(error|v \in \mathcal{R}_L)$ is upper bounded by the error rate at the point closest to the boundary. 
We choose an appropriate thickness $\tau$ of $\mathcal{R}_H$ such that the sum of the error rate of the two parts is minimized. Thus we have,
\begin{align}
\label{eqn:firstoverall}
P_e^M = & P_{e\text{ in }\mathcal{R}_H}^M + P_{e\text{ in }\mathcal{R}_L}^M \notag\\
\leq & \binom{M}{2} \cdot P(S_i, S_j \text{ switches order in }\mathcal{R}_H) + P_{e\text{ in }\mathcal{R}_L}^M \notag\\
= & \binom{M}{2} \cdot 2\int\limits_0^\tau p_D(l)P(\beta_{ij}\cdot \omega>l) \text{d}l	+ P_{e\text{ in }\mathcal{R}_L}^M \notag\\
= & \binom{M}{2} \cdot  2\int\limits_0^\tau p_D(l)Q\left( \frac{l}{\hat{\sigma} ||\beta_{ij}||_2 }\right) \text{d}l + P_{e\text{ in }\mathcal{R}_L}^M 
\end{align}

$Q(\cdot)$ is the tail probability of the standard normal distribution and is decreasing on $[0, +\infty)$. Thus for the ``low error" region, we have,

\begin{align}
\label{eqn:low}
P_{e\text{ in }\mathcal{R}_L}^M &<  P(v \in \mathcal{R}_L)\cdot Q\left( \frac{\tau}{\hat{\sigma} ||\beta_{ij}||_2 }\right) \notag\\	
&< Q\left( \frac{\tau}{\hat{\sigma} ||\beta_{ij}||_2 }\right) 
\end{align}

From Equation (\ref{eqn:beta}), (\ref{eqn:firstoverall}), and (\ref{eqn:low}), we have,
\begin{equation}
P_e^M \le  \binom{M}{2} \cdot  2\int\limits_0^\tau p_D(l)Q\left( \frac{l}{\hat{\sigma} }\right) \text{d}l + Q\left( \frac{\tau}{\hat{\sigma} }\right).
\end{equation}
\end{proof}

\begin{myth}
\label{thm:main}
For any positional ranking aggregation algorithm with $M$ candidates and $N$ voters, taking input from the $(\epsilon, \delta)$-differentially private system defined in Section \ref{sec:alg}, the ranking error rate $P_e^M(N)$  satisfies
$$
P_e^M(N) \le \binom{M}{2}\frac{M!-1}{\sqrt{2}} \tau + Q\left(\frac{\epsilon N \tau}{\sqrt{2\ln(2/\delta)}}\right), \forall \tau>0.
$$
\end{myth}

\begin{proof}
By Lemma \ref{thm:halflem},  we have,
\begin{align}\label{eqn:th1} \nonumber
P_e^M \le &  \binom{M}{2} \cdot  2\int\limits_0^\tau p_D(l)Q\left( \frac{l}{\hat{\sigma} }\right) \text{d}l + Q\left( \frac{\tau}{\hat{\sigma} }\right)\\\nonumber
\le & \binom{M}{2}  \cdot 2\int\limits_{0}^{\tau} p_D(l) Q(0) \text{d}l
 + Q\left( \frac{\tau}{\hat{\sigma} }\right)	\\
= & \binom{M}{2}  \cdot \int\limits_{0}^{\tau} p_D(l) \text{d}l +  Q\left( \frac{\tau}{\hat{\sigma} }\right)
\end{align}

By Lemma \ref{lem:decrease}, for any positional rules, $p_D(l) \le p_D(0)$. Hence we have,
\begin{align}\label{eqn:pe}
P_e^M \leq & \binom{M}{2}  \cdot \int\limits_{0}^{\tau} p_D(0) \text{d}l +  Q\left( \frac{\tau}{\hat{\sigma}}\right)	 \nonumber\\
= & \binom{M}{2}  \cdot   p_D(0) \tau +  Q\left( \frac{\tau}{\hat{\sigma}}\right)	
\end{align}

For positional rules, all hyperplanes $\mathcal{H}_{ij}$ pass through the $(M!-1)$-simplex centroid for any $i,j \in \{1,\dots,M\}$ since the profile at the centroid must be a tie for all candidates due to symmetry. From the literature in high dimensional geometry \cite{webb1996central}, we know that the largest cross section through the centroid of a regular $M!-1$-simplex is exactly the slice that contains $M!-2$ of its vertices and the midpoint of the remaining two vertices. The $(M!-2)$-measure of the cross section is $\sqrt{M!}/\left(\sqrt{2}(M!-2)!\right)$ for the probability simplex. Since the $(M!-1)$-measure of the probability simplex is $\sqrt{M!}/(M!-1)!$, we have,

\begin{equation}
\label{eqn:pd}
p_D(0) \le \frac{\sqrt{M!}/\left(\sqrt{2}(M!-2)!\right)}{\sqrt{M!}/(M!-1)!} = \frac{M!-1}{\sqrt{2}}
\end{equation}

From Equations (\ref{eqn:pe}) and (\ref{eqn:pd}), and the fact that ${\hat{\sigma}^2 = 2\ln(\frac{2}{\delta})/\epsilon ^2N^2}$, we have
\begin{align}
\label{eqn:final} \nonumber
P_e^M(N) \le&\binom{M}{2}\frac{M!-1}{\sqrt{2}} \tau + Q\left( \frac{\tau}{\hat{\sigma} }\right) \\
 = & \binom{M}{2}\frac{M!-1}{\sqrt{2}} \tau + Q\left(\frac{\epsilon N \tau}{\sqrt{2\ln(2/\delta)}}\right)
\end{align}
\end{proof}

By taking the derivative with respect to $\tau$, we can show that the right side of Equation (\ref{eqn:final}) is minimized when 
\begin{equation}
\tau = \frac{\sqrt{2\ln(2/\delta)}}{\epsilon N}\sqrt{-2\ln\frac{\sqrt{\pi\ln(2/\delta)M(M-1)(M!-1)}}{\sqrt{2}\epsilon N}}.
\end{equation}

\emph{Remark}: To better understand this upper bound, we can use a Q-function approximation to represent the result of Theorem \ref{thm:main}. It is known that 
\begin{equation}
Q(x) \le \frac{e^{-\frac{x^2}{2}}}{\sqrt{2\pi}x}, \forall x > 0.
\end{equation}
This is a good approximation when $x$ is large \cite{karagiannidis2007improved}. Thus we can rewrite Equation (\ref{eqn:final}) as
\begin{equation}
P_e^M(N) \le \binom{M}{2}\frac{M!-1}{\sqrt{2}} \tau + \frac{\sqrt{\ln(2/\hat{\sigma})}}{2\sqrt{\pi}\epsilon N\tau}e^{-\frac{(\epsilon N\tau)^2}{4\ln(2/\hat{\sigma})}}, \forall \tau>0.
\end{equation}

We can further simplify the expression by letting \\
$\tau = 2\sqrt{\ln N \ln (2/\delta)}/(\epsilon N)$:
\begin{equation}
P_e^M(N) \le \frac{1}{N} \left( \frac{\binom{M}{2} (M!-1) \sqrt{2 \ln N \ln (2/\delta)}}{\epsilon}  + \frac{1}{2\sqrt{\pi \ln N}} \right).
\label{eqn:simplebound}
\end{equation}
It is shown in (\ref{eqn:simplebound}) that the error rate goes to 0 at least as fast as $O(\frac{\sqrt{\ln N}}{N})$ for fixed $\delta,\epsilon$.

\begin{figure}[!t]
  \centering
  \includegraphics[width=0.4\textwidth]{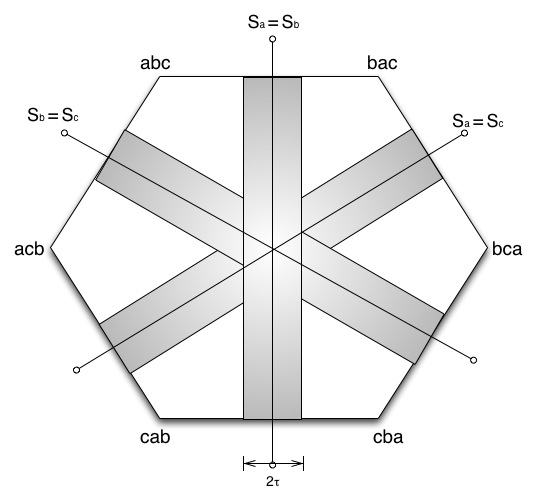}
  \caption{An example of Petrie polygon (skew orthogonal projections) of three candidates. Three hyperplanes, under Borda count ranking rule, separate the simplex into six polytopes.}
  \label{fig:example}
 \end{figure}

 \subsection{Asymptotic Error Rate}
 In this section, we analyze the asymptotic error rate for any positional ranking rule. We start by showing a tighter bound on the general error rate that can be derived from the proof of Theorem \ref{thm:main}. 
\begin{mylem}
\label{lem:tight}
An upper bound for the  ranking error rate of any $(\epsilon, \delta)$-differentially private positional ranking system with $M$ candidates and $N$ voters is
\begin{equation}\nonumber
 \binom{M}{2}\sqrt{2}(M!-1) Q\left(\frac{\epsilon N \tau}{2\sqrt{2\ln(2/\delta)}}\right) \tau  
+  Q\left(\frac{\epsilon N \tau}{\sqrt{2\ln(2/\delta)}}\right)
\end{equation}
for $\forall \tau > 0$.
\end{mylem}
\begin{proof}
Since the Q-function is convex on $[0,+ \infty)$, by Jensen's Inequality, from $\textrm{Lemma \ref{lem:decrease}}$ and  $\textrm{Lemma \ref{thm:halflem}}$,  we have
\begin{align}
\label{eqn:overall}
P_e^M(N) &\le  \binom{M}{2} \cdot  2\int\limits_0^\tau p_D(l)Q\left( \frac{l}{\hat{\sigma} }\right) \text{d}l + Q\left( \frac{\tau}{\hat{\sigma} }\right) 
\notag\\
&\le  \binom{M}{2} \cdot  2\int\limits_0^\tau p_D(0)Q\left( \frac{l}{\hat{\sigma} }\right) \text{d}l + Q\left( \frac{\tau}{\hat{\sigma} }\right)  \notag\\
&\le \binom{M}{2} \cdot  2p_D(0)Q\left( \frac{\tau}{2\hat{\sigma} }\right)  + Q\left( \frac{\tau}{\hat{\sigma}  }\right) \notag\\
&=\binom{M}{2}\sqrt{2}(M!-1) Q\left(\frac{\epsilon N \tau}{2\sqrt{2\ln(2/\delta)}}\right) \tau  \notag\\
&+  Q\left(\frac{\epsilon N \tau}{\sqrt{2\ln(2/\delta)}}\right).
\end{align}
\end{proof}

Lemma \ref{lem:tight}  slightly improves the bound in Theorem \ref{thm:main}. We use this lemma to assist the proof of the following Theorem.

\begin{myth}
\label{thm:asyn}
For any positional ranking aggregation algorithm with $M$ candidates, taking input from the $(\epsilon, \delta)$-differentially private system defined in Section \ref{sec:alg}, 
$$\lim_{N\rightarrow \infty} P_e^M(N) = 0$$ 
for any given $\epsilon$ and $\delta$. 
\end{myth}
\begin{proof}
This directly follows from $\textrm{Lemma \ref{lem:tight}}$ and the Bounded Convergence Theorem.
\end{proof}


\section{Simulation Results}
\label{sec:example}
\noindent In this section, we use Borda count with three candidates as an example. Once the ranking rule is known, we can derive a tighter bound than the general error rate bound in $\textrm{Section \ref{sec:main}}$, because we know exactly what the pairwise comparison boundaries are. We will compare all upper bounds with the simulation error rates.

In Borda count, for every vote the candidate ranked first receives 1 point, the second receives 0.5 points, and the bottom candidate receives no points. The aggregated rank is sorted according to the total points each candidate receives. We list $3!=6$ permutations in the following order, and we will stick to this order for the rest of this paper: $abc, acb, cab, cba, bca, bac$. Let 
\begin{equation} 
M = \left( \begin{array}{cccccc}
1 & 1 & 0.5 & 0 & 0 & 0.5  \\
0.5 & 0 & 0 & 0.5 & 1 & 1  \\
0 & 0.5 & 1 & 1 & 0.5 & 0 
\end{array} \right).
\end{equation} 

Then we have 
\begin{equation} 
\left(\begin{array}{c}
S_a  \\
S_b  \\
S_c 
\end{array}\right) = Mv,
\end{equation}
where $v$ is defined in Section \ref{sec:linear} and $S_a, S_b, S_c$ are the aggregated score of candidates $a,b$ and $c$ respectively. The hyperplane $\mathcal{H}_{ab}$ satisfies $S_a = S_b$,
\begin{equation}
2v_1+2v_2+v_3+v_6 = v_1+v_4+2v_5+2v_6
\end{equation}
i.e.
\begin{equation}
\label{eqn:h1}
\mathcal{H}_{ab}: v_1+2v_2+v_3-v_4-2v_5-v_6=0
\end{equation}
Similarly, we have
\begin{equation}
\label{eqn:h2}
\mathcal{H}_{bc}: v_1-v_2-2v_3-v_4+v_5+2v_6=0
\end{equation}
\begin{equation}
\label{eqn:h3}
\mathcal{H}_{ac}: 2v_1+v_2-v_3-2v_4-v_5+v_6=0
\end{equation}

With Equations (\ref{eqn:h1}), (\ref{eqn:h2}) and (\ref{eqn:h3}), we can compute the volume of the cross section made by the hyperplane cutting through the probability simplex (\ref{eqn:simplex}), using methods proposed in \cite{lawrence1991polytope}. Then an upper bound specifically for Borda count can be derived with a similar approach as Theorem \ref{thm:main} or $\textrm{Lemma \ref{lem:tight}}$. 

Figure \ref{fig:upperbound} shows the simulation results of Borda count with 3 candidates and 2,000 voters, repeated 100,000 times. We set $\delta = 5 \times 10^{-4}$ (which is 0.1 divided by the number of voters), and plot the graph of error rate with $\epsilon$ taking values between 0.05 and 0.24. We compare the simulation results with the general upper bound derived in Theorem \ref{thm:main} and the improved upper bound in Lemma \ref{lem:tight}, as well as the ranking rule-specific upper bound described above. 

\begin{figure}[!t]
  \centering
  \includegraphics[width=0.48\textwidth]{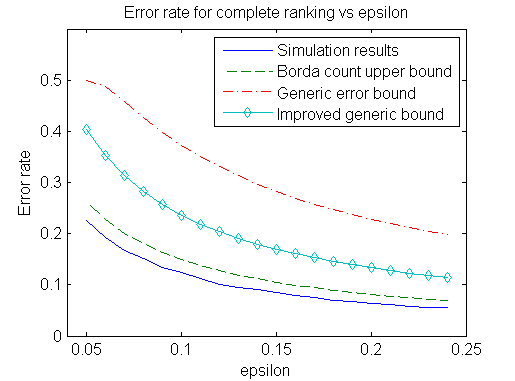}
  \caption{Error rate for vs $\epsilon$.}
  \label{fig:upperbound}
 \end{figure}

 
 Figure \ref{fig:voters} shows the simulation results for Borda count with 3 candidates with fixed $\epsilon$, repeated 20,000 times. We set $\epsilon = 0.1$ and $\delta = 0.1/N$, where $N$ is the number of voters. The number of voters varies from 1,000 to 100,000. 
 The error vanishes fast with a growing number of voters, even if we set $\delta$ to be inversely proportional to the number of voters.
 We also compare the simulation results with the general upper bound derived in Theorem \ref{thm:main} and the improved upper bound in Lemma \ref{lem:tight}, as well as the ranking rule-specific upper bound described above. 
 
\begin{figure}[!t]
  \centering
  \includegraphics[width=0.48\textwidth]{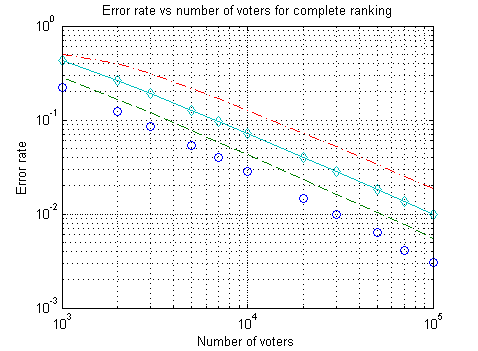}
  \caption{Error rate vs number of voters.}
  \label{fig:voters}
 \end{figure}

\section{Conclusions}
\noindent In this paper, we apply the framework of differential privacy to rank aggregation by adding noise in the votes. We analyze the probability that the aggregated ranking becomes inaccurate due to the noise and derive upper bounds on the error rates of ranking for all positional ranking rules under the assumption that profiles are uniformly distributed. The bounds can be tightened using techniques in high dimensional polytope volume computation if we are given a specific ranking rule. Our results provide insights into the trade-offs between privacy and accuracy in rank aggregation.

\section{Acknowledgments}
This research was supported in part by the Center for Science of
Information (CSoI), a National Science Foundation (NSF) Science and Technology Center, under grant
agreement CCF-0939370, by NSF under the grant CCF-1116013, by Air Force Office of Scientific Research, under the grant 
FA9550-12-1-0196, and by a research grant from Deutsche Telekom
AG.

\bibliographystyle{IEEEbib}
\bibliography{myref}

\begin{thebibliography}{10}

\bibitem{dwork2006differential}
Cynthia Dwork,
\newblock ``Differential privacy,''
\newblock in {\em Automata, languages and programming}, pp. 1--12. Springer,
  2006.

\bibitem{dwork2006our}
Cynthia Dwork, Krishnaram Kenthapadi, Frank McSherry, Ilya Mironov, and Moni
  Naor,
\newblock ``Our data, ourselves: Privacy via distributed noise generation,''
\newblock in {\em Advances in Cryptology-EUROCRYPT 2006}, pp. 486--503.
  Springer, 2006.

\bibitem{task2012guide}
Christine Task and Chris Clifton,
\newblock ``A guide to differential privacy theory in social network
  analysis,''
\newblock in {\em Proceedings of the 2012 International Conference on Advances
  in Social Networks Analysis and Mining (ASONAM 2012)}. IEEE Computer Society,
  2012, pp. 411--417.

\bibitem{mcsherry2009differentially}
Frank McSherry and Ilya Mironov,
\newblock ``Differentially private recommender systems: building privacy into
  the net,''
\newblock in {\em Proceedings of the 15th ACM SIGKDD international conference
  on Knowledge discovery and data mining}. ACM, 2009, pp. 627--636.

\bibitem{lindell2011practical}
Yehuda Lindell and Eran Omri,
\newblock ``A practical application of differential privacy to personalized
  online advertising.,''
\newblock {\em IACR Cryptology ePrint Archive}, vol. 2011, pp. 152, 2011.

\bibitem{machanavajjhala2011personalized}
Ashwin Machanavajjhala, Aleksandra Korolova, and Atish~Das Sarma,
\newblock ``Personalized social recommendations: accurate or private,''
\newblock {\em Proceedings of the VLDB Endowment}, vol. 4, no. 7, pp. 440--450,
  2011.

\bibitem{kalai2002fourier}
Gil Kalai,
\newblock ``A fourier-theoretic perspective on the condorcet paradox and
  arrow's theorem,''
\newblock {\em Advances in Applied Mathematics}, vol. 29, no. 3, pp. 412--426,
  2002.

\bibitem{conitzer2006improved}
Vincent Conitzer, Andrew Davenport, and Jayant Kalagnanam,
\newblock ``Improved bounds for computing kemeny rankings,''
\newblock in {\em AAAI}, 2006, vol.~6, pp. 620--626.

\bibitem{dwork2001rank}
Cynthia Dwork, Ravi Kumar, Moni Naor, and Dandapani Sivakumar,
\newblock ``Rank aggregation methods for the web,''
\newblock in {\em Proceedings of the 10th international conference on World
  Wide Web}. ACM, 2001, pp. 613--622.

\bibitem{dwork2006calibrating}
Cynthia Dwork, Frank McSherry, Kobbi Nissim, and Adam Smith,
\newblock ``Calibrating noise to sensitivity in private data analysis,''
\newblock in {\em Theory of Cryptography}, pp. 265--284. Springer, 2006.

\bibitem{barak2007privacy}
Boaz Barak, Kamalika Chaudhuri, Cynthia Dwork, Satyen Kale, Frank McSherry, and
  Kunal Talwar,
\newblock ``Privacy, accuracy, and consistency too: a holistic solution to
  contingency table release,''
\newblock in {\em Proceedings of the twenty-sixth ACM SIGMOD-SIGACT-SIGART
  symposium on Principles of database systems}. ACM, 2007, pp. 273--282.

\bibitem{webb1996central}
Simon Webb,
\newblock ``Central slices of the regular simplex,''
\newblock {\em Geometriae Dedicata}, vol. 61, no. 1, pp. 19--28, 1996.

\bibitem{bochner1959lectures}
Salomon Bochner,
\newblock {\em Lectures on Fourier integrals}, vol.~42,
\newblock Princeton University Press, 1959.

\bibitem{karagiannidis2007improved}
George~K Karagiannidis and Athanasios~S Lioumpas,
\newblock ``An improved approximation for the gaussian q-function,''
\newblock {\em Communications Letters, IEEE}, vol. 11, no. 8, pp. 644--646,
  2007.

\bibitem{lawrence1991polytope}
Jim Lawrence,
\newblock ``Polytope volume computation,''
\newblock {\em Mathematics of Computation}, vol. 57, no. 195, pp. 259--271,
  1991.

\end{thebibliography}

%
%
\end{document}